\PassOptionsToPackage{vlined,linesnumbered}{algorithm2e}
\documentclass[11pt]{article}
\usepackage[margin=1in]{geometry}
\usepackage{amssymb,amsthm}
\usepackage{natbib}
\usepackage{graphicx,xcolor}
\usepackage[algo2e,ruled]{algorithm2e}
\usepackage{hyperref}
\newtheorem{theorem}{Theorem}
\newtheorem*{informaltheorem}{Theorem}
\newtheorem{lemma}[theorem]{Lemma}

\theoremstyle{definition}

\theoremstyle{remark}

\newenvironment{proofof}[1]{\begin{proof}[Proof of #1]}{\end{proof}}

\usepackage{mathtools}
\usepackage{microtype}
\usepackage{times}
\PrintSemicolon
\definecolor{algorithmcomment}{HTML}{483D8B}

\SetCommentSty{algorithmcommentstyle}
\hypersetup{colorlinks=true,citecolor=blue,linkcolor=blue,urlcolor=blue}

\counterwithin{theorem}{section}

\newcommand{\OPT}{\operatorname{OPT}}
\newcommand{\ALG}{\mathrm{ALG}}
\newcommand{\Reg}{\operatorname{Reg}}
\newcommand{\E}{\mathbb{E}}
\newcommand{\ind}{\mathbf{1}}

\title{Simultaneous Competitive and Regret Guarantees\\
for TCP Acknowledgment}
\author{Tsubasa Harada\\Institute of Science Tokyo
\and Yasushi Kawase\\Chuo University
\and Hanna Sumita\\Institute of Science Tokyo}
\date{}

\begin{document}

\maketitle

\begin{abstract}
We study simultaneous competitive and regret guarantees for TCP acknowledgment.
Each acknowledgment serves all pending packets, and the objective is to minimize the total acknowledgment and waiting costs.
It is well known that no deterministic online algorithm achieves a competitive ratio better than $2$.
Given $k$ randomized online experts and $0<\varepsilon<1$, we construct a randomized online algorithm whose realized total cost is always at most $(2+\varepsilon)$ times the offline optimal cost $\OPT$ and whose expected regret against the expert with the smallest realized cost is $O(\sqrt{\OPT\log(k+1)}+\log(1/\varepsilon))$.
The algorithm requires no advance knowledge of the number of packets or $\OPT$.
Our algorithm controls switching costs between experts and uses the gap between the permitted competitive bound and its current cost to determine when to learn.
We also prove lower bounds showing that, for $k\ge 2$, this regret bound is asymptotically optimal in $\OPT$, $k$, and $\varepsilon$, and that $\varepsilon>0$ is necessary for sublinear regret.
\end{abstract}


\section{Introduction}

In TCP acknowledgment, an online algorithm decides when to acknowledge arriving packets.
Each acknowledgment costs one and serves all packets that have arrived but have not yet been acknowledged.
Delaying an acknowledgment can reduce the number of acknowledgments, but increases the waiting cost of the packets.
The objective is to minimize the number of acknowledgments plus the total waiting cost.

TCP acknowledgment has been studied through competitive analysis, which compares an online algorithm with an optimal offline schedule on the same input.
\citet{dooly-goldman-scott} show that the greedy algorithm, which acknowledges whenever the waiting cost accumulated since its previous acknowledgment reaches one, is $2$-competitive, and that no deterministic online algorithm achieves a smaller competitive ratio.
A worst-case competitive guarantee alone, however, does not ensure that an algorithm performs nearly as well as a policy suited to the particular arrival pattern.
Different acknowledgment policies may perform well on different inputs.

This motivates an approach based on prediction with expert advice.
We are given a collection of online acknowledgment algorithms as experts and use their observed costs to decide which expert to follow.
In addition to the competitive ratio, we evaluate the resulting algorithm by its regret: the difference between its total cost and the smallest total cost incurred by an expert on the same input.
Writing $\OPT$ for the offline optimal cost, an expected regret bound of $o(\OPT)$ ensures that the algorithm's expected cost is at most $1+o(1)$ times the expected cost of the best expert in hindsight as $\OPT$ grows.

Our goal is to adapt to the experts without giving up the worst-case guarantee of competitive analysis.
Specifically, we ask whether an online algorithm can maintain a competitive ratio arbitrarily close to two for every realization of its random choices while achieving expected regret sublinear in the offline optimal cost.

\subsection{Our results}

We construct a randomized online algorithm whose total cost is at most $(2+\varepsilon)\OPT$ on every input and for every realization of the random choices, while achieving expected regret sublinear in $\OPT$ for fixed $k$ and $\varepsilon$.
Thus, the algorithm can learn from the given experts while keeping its total cost within $2+\varepsilon$ times the offline optimal cost.

\begin{informaltheorem}[Upper bound; see Theorem~\ref{thm:simultaneous-guarantees}]
For every $0<\varepsilon<1$ and every family of $k\ge1$ randomized online experts for TCP acknowledgment, there is a randomized online algorithm $\ALG$ such that
$C_{\ALG}\le(2+\varepsilon)\OPT$
for every input and every realization of the random choices, and
\begin{align*}
 \E\left[C_{\ALG}-\min_{i\in[k]}C_{A_i}\right]
 &\le O\left(\sqrt{\OPT\log(k+1)}+\log\frac1\varepsilon\right).
\end{align*}
The algorithm requires no prior knowledge of the number of packets or $\OPT$, and uses the polynomial number of arithmetic operations with respect to $k$ and the number of packets.
\end{informaltheorem}

The regret bound depends on $\varepsilon$ only through an additive $O(\log(1/\varepsilon))$ term.
Allowing the total cost to exceed $2\OPT$ is necessary for sublinear expected regret.
We show that any randomized online algorithm whose total cost is at most $2\OPT$ on every input and for every realization of the random choices must incur expected regret $\Omega(\OPT)$ on some inputs.
For $k\ge2$, we also establish lower bounds matching the dependence on $\OPT$, $k$, and $\varepsilon$ in the upper bound.

\begin{informaltheorem}[Lower bounds; see Section~\ref{sec:lower-bounds}]
For every $k\ge2$, there exists a family $\mathcal{E}_k$ of at most $k$ deterministic online experts with the following property.
Let $0<\varepsilon<1$, and let $\ALG$ be any randomized online algorithm whose total cost satisfies
$C_{\ALG}\le(2+\varepsilon)\OPT$
on every input and for every realization of the random choices.
Then, for every integer $h\ge1$, there is an input with $h\le\OPT\le3h/2$ on which
\begin{align*}
 \E[C_{\ALG}]-\min_{A\in\mathcal{E}_k}C_A
 &\ge
 \Omega\left(
 \min\left\{
 \OPT,
 \sqrt{\OPT\log k}+\log\frac1\varepsilon
 \right\}
 \right).
\end{align*}
Moreover, every randomized online algorithm whose total cost is at most $2\OPT$ on every input and for every realization of the random choices incurs expected regret $\Omega(\OPT)$ on inputs with arbitrarily large $\OPT$.
\end{informaltheorem}

\subsection{Overview of techniques}

\paragraph{Learning with switching costs.}
We first divide continuous time into rounds in which each expert incurs bounded cost.
Switching experts requires some care because the learner and the newly selected expert may have different unacknowledged packets.
Simply copying the expert's acknowledgment times may therefore cause the learner to incur more waiting cost than the expert.
To avoid this, the learner sends one additional acknowledgment when necessary at a switch.
Its unacknowledged packets then form a subset of those of the selected expert, and this inclusion is preserved while the learner follows that expert.
Consequently, the learner's cost in each round is at most the selected expert's cost, apart from the possible additional acknowledgment at the start.
This reduces the learning problem to prediction with expert advice with switching cost at most one.
We use the multiplicative-weights method of \citet{geulen-voecking-winkler}, which correlates successive expert choices to control the number of switches.

\paragraph{Maintaining the competitive guarantee.}
We maintain the nonnegative budget
$\Phi(t)=(2+\varepsilon)\OPT(t)-C_{\ALG}(t)$,
where $\OPT(t)$ is the minimum offline cost up to time $t$, allowing packets to remain unacknowledged and counting their waiting cost only up to $t$.
Thus, $\Phi(t)$ is the additional cost that the algorithm can incur without violating the competitive bound at that time.

We run the learner on time intervals called blocks, whose sizes are proportional to the budget available at their starts.
Within each block, we combine the learner with a randomized threshold algorithm from the family of \citet{karlin-kenyon-randall}.
If the learner's cumulative cost becomes too large relative to that of the threshold algorithm, we switch to the latter.
This limits the realized cost of a block while preserving a regret bound against each expert.
Moreover, the threshold algorithm has a competitive ratio strictly below two in expectation.
We use this advantage to show that completed blocks increase the budget in expectation and to bound the sum of their regret contributions.

When the budget is too small to start a learning block, we instead run a separate procedure that increases it by at least one before returning.
This procedure adaptively randomizes between two thresholds and has expected cost $O(\log(1/\varepsilon))$ per invocation.
We show that it is invoked only $O(1)$ times in expectation, and hence its total contribution to regret is also $O(\log(1/\varepsilon))$.

\subsection{Related work}

\paragraph{Competitive analysis of TCP acknowledgment.}
\citet{dooly-goldman-scott} provide a deterministic $2$-competitive algorithm that acknowledges whenever the accumulated waiting cost since its previous acknowledgment reaches one, and establish the optimality of this competitive ratio.
\citet{seiden} proves a lower bound of $e/(e-1)$ for randomized online algorithms.
\citet{karlin-kenyon-randall} attain this ratio using a different threshold family: they sample a single threshold at the start and use it throughout the execution.
\citet{basiak-et-al} study the tradeoff between expected competitiveness and competitive guarantees that hold for every realization of the random choices.
Several extensions of TCP acknowledgment have also been studied \citep{albers-bals,bienkowski-et-al,kawase-nakayoshi,bhore-pawlowski-umboh}.

\citet{imreh-nemeth} adapt the acknowledgment threshold using past arrivals and evaluate the resulting algorithm experimentally.
\citet{im-moseley-xu-zhang} use predictions of the arrival sequence to obtain consistency and robustness guarantees.
\citet{bamas-maggiori-svensson} give a multiplicative comparison with a predicted schedule through a learning-augmented primal-dual method.
Our learning guarantee is an additive comparison with the best online expert, with no assumption on prediction accuracy.

\paragraph{Simultaneous competitive and regret guarantees.}
Whether a single algorithm can achieve both guarantees depends on the problem and the benchmarks used.
\citet{andrew-et-al} show that a constant competitive ratio and sublinear static regret can be incompatible in smoothed online convex optimization.
For metrical task systems, \citet{daniely-mansour} retain a baseline algorithm's competitive ratio in expectation while obtaining regret bounds on every interval.
Their competitive guarantee permits an additive constant, whereas ours bounds realized total cost by a multiple of $\OPT$ and provides regret bounds for the full input.
\citet{goel-et-al} obtain sublinear policy regret and an optimal competitive ratio up to a sublinear additive term in online linear control.
\citet{harada-kawase-sumita} study the compatibility of per-round baseline guarantees and sublinear regret in repeated optimal stopping.

\paragraph{Combining online algorithms.}
Expert-learning methods have also been used to combine online algorithms, with the goal of performing nearly as well as the best constituent algorithm.
\citet{blum-burch} apply expert-learning methods to combine $k$ online algorithms for a metrical task system of diameter $D$, obtaining expected cost at most $(1+\eta)$ times the cost of the best algorithm plus $O(D\log k/\eta)$.
\citet{buchbinder-et-al} give a unified algorithm that interpolates, through parameter tuning, between regret guarantees for expert learning and competitive guarantees for metrical task systems.
Algorithm combination also appears in work on untrusted predictions and comparison with a sequence of predictors \citep{antoniadis-et-al-2020,antoniadis-et-al-2023}.

\section{Model}
\label{sec:model}

An input $\sigma$ consists of packets indexed by $[n]\coloneqq\{1,\ldots,n\}$, arriving online at times $0\le a_1\le\cdots\le a_n$.
The algorithm is not informed when the input ends and does not know $n$ in advance.
At each time, the algorithm first observes all packets arriving at that time.
If any packets are pending, it may then send an acknowledgment, which costs one and serves all pending packets.
Let $b_1<\cdots<b_m$ be the acknowledgment times of an algorithm $A$.
Each of the intervals $[0,b_1]$ and $(b_{i-1},b_i]$ for $i=2,\ldots,m$ is called an \emph{acknowledgment interval}.
The algorithm must serve all packets (i.e., $b_m\ge a_n$), and the waiting cost of packet $r\in[n]$ is $\min_{i\in[m]:\,b_i\ge a_r}(b_i-a_r)$.

Let $Q_A(t)\subseteq[n]$ be the set of packets pending under algorithm $A$ after processing all arrivals and any acknowledgment at time $t$.
Since each pending packet accumulates waiting cost at unit rate, the waiting cost incurred over $[v,t]$ is $\int_v^t|Q_A(u)|\,du$ for $0\le v\le t$.
The total cost of $A$ on $\sigma$ is
\begin{align*}
 C_A(\sigma)&=m+\sum_{r\in[n]}\min_{i\in[m]:\,b_i\ge a_r}(b_i-a_r)
 =m+\int_0^\infty|Q_A(u)|\,du.
\end{align*}
Let $\OPT(\sigma)$ denote the minimum total cost achievable with the input known in advance.
We use $C_A(\sigma,t)$ for the total acknowledgment and waiting cost incurred up to time $t$, including any acknowledgment at $t$.
We write $C_A(\sigma,t^-)$ for the cost immediately before the arrivals and acknowledgment decision at $t$, with $C_A(\sigma,0^-)=0$.
Thus, $C_A(\sigma,t)-C_A(\sigma,t^-)$ is the acknowledgment cost at time $t$.
When the input is clear, we abbreviate the total costs as $C_A$ and $\OPT$ and omit $\sigma$ from the time-dependent quantities.
For a nonempty input, $1\le\OPT\le n$, since acknowledging immediately at each distinct arrival time costs at most $n$.

We are given $k$ randomized online algorithms $A_1,\ldots,A_k$ as experts.
We assume that an algorithm $\ALG$ can simulate each expert online on the same input, observing its acknowledgment decisions and realized costs before making its own acknowledgment decision at each time.
We consider an oblivious adversary: the input is fixed independently of all random choices.
The experts' random choices are independent of those of $\ALG$, and expectations are taken over both unless stated otherwise.

The algorithms we construct combine multiple algorithms.
At each time $t$, the following order applies, where $A$ is the algorithm being followed immediately before processing time $t$:
(1) receive all packets arriving at $t$;
(2) if $A$ acknowledges at $t$ and any packets are pending, send an acknowledgment;
(3) determine whether to switch to another algorithm; and
(4) if a switch is made, acknowledge any pending packets and then follow the new algorithm.

For $\alpha\ge1$, our competitive guarantee requires the realized total cost to satisfy $C_{\ALG}(\sigma)\le\alpha\OPT(\sigma)$ for every input $\sigma$ and every possible realization of the random choices.
We define expected regret as $\Reg(\ALG)\coloneqq\E[C_{\ALG}-\min_{1\le i\le k}C_{A_i}]$, comparing the algorithm with the expert having the smallest realized total cost on the same input.
Since $\E[\min_{1\le i\le k}C_{A_i}]\le\min_{1\le i\le k}\E[C_{A_i}]$, $\Reg(\ALG)$ also bounds $\E[C_{\ALG}]-\min_{1\le i\le k}\E[C_{A_i}]$.

\section{Simultaneous competitive and regret guarantees}
\label{sec:main-result}

We provide an algorithm with a $(2+\varepsilon)$ competitive guarantee and expected regret sublinear in $\OPT$ for fixed $k$ and $\varepsilon$.

\begin{theorem}
\label{thm:simultaneous-guarantees}
For every $\varepsilon\in(0,1)$ and every family of $k\ge1$ randomized online algorithms for TCP acknowledgment, there is a randomized online algorithm $\ALG$ such that, on every input, its realized total cost is always at most $(2+\varepsilon)\OPT$ and its expected regret satisfies
\begin{align}
 \Reg(\ALG)&\le O\left(\sqrt{\OPT\log(k+1)}+\log\frac1\varepsilon\right).
 \label{eq:regret-upper-bound}
\end{align}
\end{theorem}

We first recall the greedy algorithm of \citet{dooly-goldman-scott} and establish cost bounds for the threshold algorithms of \citet{karlin-kenyon-randall} in Section~\ref{sec:rectangle-thresholds}.
In Section~\ref{sec:learning}, we construct a learner using only the given experts.
Section~\ref{sec:recovery} introduces the available budget, which measures how much more we can spend under the competitive guarantee, and gives a procedure for increasing it.
Section~\ref{sec:competitive-guarantee} combines these components by switching between the learner and a randomized threshold algorithm on blocks sized according to the budget.

\subsection{Competitive algorithms}
\label{sec:rectangle-thresholds}

Let $\OPT(t)$ be the minimum cost through time $t$ of an offline schedule that may leave packets pending, including their waiting cost through $t$.
It depends only on packets received by time $t$.
The function $\OPT(t)$ is continuous and nondecreasing in $t$, and $\OPT(t)=\OPT$ for all sufficiently large $t$.

The greedy algorithm $A_{\mathrm{G}}$ acknowledges whenever the waiting cost accumulated since its previous acknowledgment reaches one.
Its standard competitive analysis \citep{dooly-goldman-scott} shows $C_{A_{\mathrm{G}}}(t)\le2\OPT(t)$ for every input and time $t$.

For $0<z\le1$, let $A_{\mathrm{R}}(z)$ be the algorithm of \citet{karlin-kenyon-randall} that acknowledges as soon as there is an integer $j\ge1$ for which at least $j$ pending packets have each waited at least $z/j$ time units.
For $z=0$, we define $A_{\mathrm{R}}(0)$ to acknowledge each packet immediately upon arrival.

\begin{lemma}
\label{lem:two-thresholds}
Let $N_z(t)$ be the number of acknowledgments sent by $A_{\mathrm{R}}(z)$ through time $t$.
For every input, time $t$, and $z\in(0,1]$, we have $N_1(t)\le N_z(t)$ and
\begin{align*}
 zN_z(t)&\le C_{A_{\mathrm{R}}(z)}(t)-N_z(t)\le\OPT(t)-(1-z)N_1(t).
\end{align*}
\end{lemma}

The proof is given in Appendix~\ref{app:proof-two-thresholds}.

The next lemma bounds the cost for each threshold and the expected cost of randomizing between two thresholds, and provides the inequalities used to analyze the available budget in Section~\ref{sec:recovery}.

\begin{lemma}
\label{lem:random-comparison}
For every input, time $t$, and $u\in(0,1]$, we have
\begin{align*}
 C_{A_{\mathrm{R}}(u)}(t)&\le\min\bigl\{\OPT(t)+N_u(t),\ 2\OPT(t)-(1-u)(2N_1(t)-N_u(t))\bigr\}
 \le\left(1+\frac1u\right)\OPT(t).
\end{align*}
Moreover, for any $x\in(0,1)$, let $Z$ equal $1$ or $x$, each with probability $1/2$. Then,
\begin{align*}
 \E[C_{A_{\mathrm{R}}(Z)}(t)]&\le\left(1+\max\left\{\frac1{2x},\ \frac{1+x}{2}\right\}\right)\OPT(t).
\end{align*}
\end{lemma}

The proof is given in Appendix~\ref{app:proof-random-comparison}.

\subsection{Learning with switching costs}
\label{sec:learning}

Algorithm~\ref{alg:learning} constructs the learner $A_{\mathrm{L}}$ from the given experts $A_1,\ldots,A_k$, simulated online.

Let $\mathcal{T}_{\mathrm{ack}}$ be the set of times at which at least one expert sends an acknowledgment.
A round ends when an expert acknowledges or when some expert has incurred one unit of waiting cost during that round.
For convenience, set $\tau_0=0^-$, with $t>0^-$ interpreted as $t\ge0$ and integrals from $0^-$ evaluated from $0$.
For $s\ge0$, define the end of round $s+1$ by
\begin{align*}
 \tau_{s+1}
 &=\inf\left\{t>\tau_s:
 t\in\mathcal{T}_{\mathrm{ack}}
 \ \text{or}\ 
 \max_{i\in[k]}\int_{\tau_s}^t|Q_{A_i}(u)|\,du=1\right\},
\end{align*}
where $\inf\emptyset=\infty$.
These boundaries are detected online by simulating the experts.

Round $s\ge1$ is $(\tau_{s-1},\tau_s]$, with $(0^-,t]\coloneqq[0,t]$.
In particular, acknowledgments at $\tau_s$ are included in round $s$.
Define the expert losses by
$\ell_s(i)=C_{A_i}(\tau_s)-C_{A_i}(\tau_{s-1})$ for $s\ge1$.
Each expert incurs at most one unit of waiting cost and at most one acknowledgment cost in a round.
Hence, $\ell_s\in[0,2]^k$.

We call the acknowledgment sent at an expert switch in Algorithm~\ref{alg:learning} synchronization.
Excluding synchronization, the learner's cost in each round is at most the selected expert's cost.
Each switch therefore adds at most one to the cost of following the experts.

The weight updates and retention/redraw rule in Algorithm~\ref{alg:learning} implement the multiplicative-weights method of \citet{geulen-voecking-winkler} with normalized losses $\ell_s(i)/2$.

\begin{algorithm2e}[htbp]
\caption{Learning with switching costs ($A_{\mathrm{L}}$)}
\label{alg:learning}
\KwIn{Experts $A_1,\ldots,A_k$ and fixed learning rate $0<\eta\le1/2$}
Set $w_{i,1}\gets1/k$ for $i\in[k]$ and draw $I_1$ uniformly from $[k]$\;
\For{$s=1,2,\ldots$}{
 Follow expert $A_{I_s}$ during round $s$, skipping acknowledgments when no packets are pending, and observe $\ell_s\in[0,2]^k$\;
 \tcp{Time $\tau_s$}
 Set $w_{i,s+1}\gets w_{i,s}(1-\eta)^{\ell_s(i)/2}$ for all $i\in[k]$\;
 With probability $w_{I_s,s+1}/w_{I_s,s}$, set $I_{s+1}\gets I_s$\;
 Otherwise, set $I_{s+1}\gets i$ with probability $w_{i,s+1}/\sum_{j=1}^{k}w_{j,s+1}$ for $i\in[k]$\;
 If $I_{s+1}\ne I_s$, acknowledge any pending packets\;
}
\end{algorithm2e}

\begin{lemma}
\label{lem:expected-switching-learner}
Fix the input and the experts' random choices.
For any fixed number $h$ of completed rounds, let $Y$ be the learner's cost through the end of round $h$, and let $L_i=\sum_{s=1}^h\ell_s(i)$ be expert $A_i$'s cost over these rounds.
Then, for every $i\in[k]$, it holds that
\begin{align}
 \E[Y]-L_i&\le4\eta L_i+\frac{4\ln k}{\eta}.
 \label{eq:expected-switching-learner}
\end{align}
The expectation is taken over the learner's expert selections.
\end{lemma}

The proof is given in Appendix~\ref{app:proof-expected-switching-learner}.

\subsection{Increasing the budget}
\label{sec:recovery}

We maintain the competitive guarantee while using $A_{\mathrm{L}}$ by keeping $C_{\ALG}(t)\le(2+\varepsilon)\OPT(t)$ at all times.
The \emph{available budget} $\Phi(t)\coloneqq(2+\varepsilon)\OPT(t)-C_{\ALG}(t)$ is the additional cost we can pay without exceeding this bound.
Algorithm~\ref{alg:recovery} starts with no pending packets and increases this budget by at least one when it returns.
A simple choice is to run a fresh execution of $A_{\mathrm{G}}$: each completed acknowledgment interval costs two and increases the budget by at least $\varepsilon$.
Thus, $\lceil1/\varepsilon\rceil$ completed acknowledgment intervals suffice to increase the budget by at least one, at a total cost of $2\lceil1/\varepsilon\rceil=O(1/\varepsilon)$.
Algorithm~\ref{alg:recovery} instead uses the algorithms $A_{\mathrm{R}}(u)$ from Section~\ref{sec:rectangle-thresholds} and achieves an expected total cost of $O(\log(1/\varepsilon))$.

For $t\ge t_0$, let $\OPT^{t_0}(t)$ be the minimum offline cost through $t$ restricted to packets arriving in $(t_0,t]$.
We use the same superscript for costs and acknowledgment counts of a fresh algorithm execution on these packets.
For the initial execution, let $(0^-,t]\coloneqq[0,t]$ and $\OPT^{0^-}(t)\coloneqq\OPT(t)$ for $t\ge0$, with $\OPT(0^-)\coloneqq0$.
For any offline schedule through $t$, its cost through $t_0$ is at least $\OPT(t_0)$.
Its cost after $t_0$ is at least $\OPT^{t_0}(t)$, since restricting this part of the schedule to packets arriving after $t_0$ can only decrease its cost.
Consequently,
\begin{align}
 \OPT(t)&\ge\OPT(t_0)+\OPT^{t_0}(t).
 \label{eq:partial-superadditive}
\end{align}

Algorithm~\ref{alg:recovery} specifies the trials and their stopping conditions.
Let $t_0$ be the starting time of the algorithm and $t_j$ the end of trial $j$, after any final acknowledgment.
The threshold choices use independent fair coins, and the budget stopping condition is monitored continuously as time passes.

\begin{algorithm2e}[htbp]
\caption{Budget increase by one}
\label{alg:recovery}
\KwIn{Parameter $0<\varepsilon<1$}
Set $\Phi_{\mathrm{start}}\gets\Phi(t)$ and $j\gets1$\;
\While{$\Phi(t)<\Phi_{\mathrm{start}}+1$}{
 \tcp{Time $t_{j-1}$; no packets are pending}
 Set $\beta_j\gets\Phi(t)-\Phi_{\mathrm{start}}$\;
 Set $u_j\gets1$ with probability $1/2$ and $u_j\gets1-\beta_j/4$ otherwise\;
 Start a fresh execution of $A_{\mathrm{R}}(u_j)$ on subsequent packets and run it until it has sent three acknowledgments or $\Phi(t)\ge\Phi_{\mathrm{start}}+2$\;
 Acknowledge any pending packets\;
 \tcp{Time $t_j$}
 Set $j\gets j+1$\;
}
\end{algorithm2e}

Algorithm~\ref{alg:recovery} may never return if all packets have arrived and been acknowledged during a trial without either stopping condition being met.
The cost bound in Lemma~\ref{lem:recovery} also covers this case.

\begin{lemma}
\label{lem:recovery}
Suppose a call to Algorithm~\ref{alg:recovery} starts at time $t_0$ with no pending packets.
Throughout this call, $\Phi(t)\ge\Phi(t_0)$.
If the call returns, its pending set is empty and the budget is at least $\Phi(t_0)+1$.
Fix the remaining input and condition on the history at $t_0$.
The expected total cost incurred by this call, summed over all its trials and including any final acknowledgments, is $O(\log(1/\varepsilon))$.
\end{lemma}

\begin{proof}
Fix the input and the history at the start of trial $j$, and write $z=1-\beta_j/4$.
The cost and acknowledgment count of $A_{\mathrm{R}}(u)$ started at $t_{j-1}$ on subsequent packets are denoted by $C_{A_{\mathrm{R}}(u)}^{t_{j-1}}(t)$ and $N_u^{t_{j-1}}(t)$.
For either choice $u\in\{1,z\}$, \eqref{eq:partial-superadditive} and Lemma~\ref{lem:random-comparison} yield, before any final acknowledgment,
\begin{align}
 \Phi(t)-\Phi_{\mathrm{start}}
 &\ge\beta_j+(2+\varepsilon)\OPT^{t_{j-1}}(t)-C_{A_{\mathrm{R}}(u)}^{t_{j-1}}(t) \notag\\
 &\ge\beta_j+\varepsilon\OPT^{t_{j-1}}(t)+(1-u)\bigl(2N_1^{t_{j-1}}(t)-N_u^{t_{j-1}}(t)\bigr).
 \label{eq:recovery-budget}
\end{align}
Here, $\Phi(t)$ is evaluated along the execution with threshold $u$.
For $u=1$, the budget increase stays at least $\beta_j$.
For $u=z$, the bound $N_z^{t_{j-1}}(t)\le3$ keeps it at least $\beta_j-3(1-z)=\beta_j/4$.
Any final acknowledgment is needed only when the increase reaches two, and leaves at least one.
Starting with $\beta_1=0$, this proves that $\Phi$ never falls below $\Phi_{\mathrm{start}}$ and increases by at least one on return.

Before the trial stops, Lemma~\ref{lem:random-comparison} implies $C_{A_{\mathrm{R}}(u_j)}^{t_{j-1}}(t)\le\OPT^{t_{j-1}}(t)+3$ and hence $(1+\varepsilon)\OPT^{t_{j-1}}(t) \le\Phi(t)-\Phi_{\mathrm{start}}-\beta_j+3<5$.
Since $\OPT^{t_{j-1}}(t)$ is continuous, $\OPT^{t_{j-1}}(t_j)\le5$ holds at the ending time $t_j$ of a terminating trial.
The trial's waiting cost is therefore at most five, and it sends at most three acknowledgments, including any final one.
Thus, each trial costs at most eight.

For $0<\beta_j<1$, we show that the expected number of trials from trial $j$ onward, including trial $j$ itself, is at most $c\ln(8/\beta_j)$ for a sufficiently large fixed constant $c$.
Let $m$ be the number of packets arriving after $t_{j-1}$.
We prove this claim by induction on $m$, for every possible history at the start of a trial with $0<\beta_j<1$.
If $m=0$, the trial cannot reach a third acknowledgment.
Thus, at most one trial is run, and $1\le c\ln(8/\beta_j)$ proves the base case.

Now, let $m\ge1$ and assume that the claim holds whenever fewer than $m$ packets remain.
If trial $j$ is followed by another trial, then it must stop at its third acknowledgment and hence serves at least one of the $m$ remaining packets.
The number of packets arriving after $t_j$ is therefore strictly smaller than $m$, and hence the induction hypothesis applies at the start of trial $j+1$.
Compare the two possible threshold choices $u=1$ and $u=z$ for trial $j$, on the same remaining input and with the same starting history.
If at most one of these choices leads to another trial, then another trial occurs with probability at most $1/2$ and has $\beta_{j+1}\ge\beta_j/4$.
The induction hypothesis then bounds the expected number of trials, including the current one, by
$1+\frac{c}{2}\ln\frac{32}{\beta_j} \le c\ln\frac{8}{\beta_j}$.

Suppose both choices $u=1$ and $u=z$ lead to another trial.
Under either choice, trial $j$ ends at its third acknowledgment with budget increase below one.
Let $t_j^z$ be the time at which trial $j$ would end if $u_j=z$, and let $\Phi_1(t)$ denote the budget during trial $j$ if instead $u_j=1$.
Since $N_1^{t_{j-1}}(t)\le N_z^{t_{j-1}}(t)$, the trial with threshold one ends no earlier than $t_j^z$ and has $\Phi_1(t_j^z)-\Phi_{\mathrm{start}}<2$.
Lemma~\ref{lem:two-thresholds} implies $\OPT^{t_{j-1}}(t_j^z)\ge zN_z^{t_{j-1}}(t_j^z)=3z$.
Apply \eqref{eq:recovery-budget} to the execution with threshold one.
Together with $C_{A_{\mathrm{R}}(1)}^{t_{j-1}}(t_j^z)\le\OPT^{t_{j-1}}(t_j^z)+N_1^{t_{j-1}}(t_j^z)$, we have
\begin{align*}
 N_1^{t_{j-1}}(t_j^z)
 &\ge\beta_j+\OPT^{t_{j-1}}(t_j^z)-\bigl(\Phi_1(t_j^z)-\Phi_{\mathrm{start}}\bigr)
 >\beta_j+3z-2=1+\frac{\beta_j}{4}>1.
\end{align*}
Thus, $N_1^{t_{j-1}}(t_j^z)\ge2$.
Applying \eqref{eq:recovery-budget} to the execution with threshold $z$ gives $\beta_{j+1}\ge\beta_j+(1-z)(2\cdot2-3)=5\beta_j/4$ if $u_j=z$.
If $u_j=1$, the same inequality gives $\beta_{j+1}\ge\beta_j$.
The induction hypothesis for the two resulting histories therefore bounds the expected number of trials by
$1+\frac{c}{2}\ln\frac{8}{\beta_j} +\frac{c}{2}\ln\frac{32}{5\beta_j} \le c\ln\frac{8}{\beta_j}$,
where $(c/2)\ln(5/4)\ge1$ suffices.
This completes the induction.

The first trial uses threshold one.
If it leads to another trial, its three acknowledgments imply $\OPT^{t_0}(t_1)\ge3$, and \eqref{eq:recovery-budget} gives $\beta_2\ge3\varepsilon>0$.
The expected total number of trials is therefore $O(\log(1/\varepsilon))$, proving the cost bound.
\end{proof}

\subsection{The full algorithm}
\label{sec:competitive-guarantee}

Starting with no pending packets, $\ALG$ follows Algorithm~\ref{alg:competitive} with $\theta=2^{10}$, using Algorithm~\ref{alg:recovery} to ensure sufficient budget before each block.
For a block of size $b$ starting at $t_0$, define $t_0'=\inf\{t>t_0:\OPT^{t_0}(t)=b\}$, where $\inf\emptyset=\infty$.
The block is $(t_0,t_0']$, including all arrivals and acknowledgments at $t_0'$.
If $t_0'=\infty$, then it is the final, incomplete block $(t_0,\infty)$ and runs on all remaining input.


As budget recovery and completed blocks leave no pending packets, $\ALG$ and the fresh simulations of $A_{\mathrm{L}}$ and $A_{\mathrm{R}}(Z)$ start each block with empty pending sets.
Both simulations run throughout the block, including while $\ALG$ follows the other one.
The experts continue their executions from time zero, while the learner starts its first round at \(t_0\) and uses their subsequent cost increments as losses.
Each block uses fresh, mutually independent randomness for $Z$, $\xi$, and $A_{\mathrm{L}}$.
Write $C_{A_{\mathrm{L}}}^{t_0}(t)$ and $C_{A_{\mathrm{R}}(Z)}^{t_0}(t)$ for their respective costs since $t_0$, and define their cost difference by
$D(t)=C_{A_{\mathrm{L}}}^{t_0}(t)-C_{A_{\mathrm{R}}(Z)}^{t_0}(t)$.
These are the costs of the two simulations, not the cost of $\ALG$.

When switching between simulations, $\ALG$ synchronizes by acknowledging any pending packets; while following either simulation, it skips acknowledgments when its own pending set is empty.


\begin{algorithm2e}[htbp]
\caption{Maintaining the competitive guarantee}
\label{alg:competitive}
\KwIn{Experts $A_1,\ldots,A_k$ and parameter $0<\varepsilon<1$}
\While{true}{
 \lWhile(\tcp*[f]{Increase the budget}){$\Phi(t)<\theta^2$}{\nllabel{line:increase-budget}Run Algorithm~\ref{alg:recovery}}
 Set $t_0\gets t$ and $b\gets\Phi(t_0)/\theta$\tcp*{Beginning of the block; no packets are pending}
 Set $\eta\gets\min\{1/2,\sqrt{\ln(k+1)/b}\}$ and independently draw $Z$ uniformly from $\{2/3,1\}$\;
 Independently draw $\xi$ uniformly from $[0,b/8]$\;
 Start fresh simulations of $A_{\mathrm{L}}$ with learning rate $\eta$ and $A_{\mathrm{R}}(Z)$ at $t_0$, without restarting the experts\;
 \While{$\OPT^{t_0}(t)<b$}{
  \If{$\OPT^{t_0}(t)<b$ and $D(t)<\xi+1$}{
   Synchronize with $A_{\mathrm{L}}$ as described above and follow its simulation until $D(t)\ge\xi+1$ or $\OPT^{t_0}(t)=b$\;
  }
  \If{$\OPT^{t_0}(t)<b$ and $D(t)>\xi$}{
   Synchronize with $A_{\mathrm{R}}(Z)$ as described above and follow its simulation until $D(t)\le\xi$ or $\OPT^{t_0}(t)=b$\;
  }
 }
 Acknowledge any remaining packets\tcp*{Time $t=t_0'$; the block ends}
}
\end{algorithm2e}

Consider a complete block of size $b$ starting at $t_0$ and ending at $t_0'$, and define
$X=C_{\ALG}(t_0')-C_{\ALG}(t_0)$ and $L_i=C_{A_i}(t_0')-C_{A_i}(t_0)$ for $i\in[k]$.
Thus, $X$ includes all acknowledgments sent at switches and at the block endpoint.
The expert $A_i$ is evaluated in its execution from time zero: $L_i$ includes the waiting cost accrued during the block by packets already pending under $A_i$ at $t_0$, even if they arrived at or before $t_0$.

\begin{lemma}
\label{lem:learning-block}
Fix the input and the realizations of the given experts, and condition on the history up to the start time $t_0$ of a complete block of size $b=\Phi(t_0)/\theta$ in Algorithm~\ref{alg:competitive}.
We have $X\le16b$ for every realization and $\E[X]\le(2-1/36)b$.
Moreover, $\E[X-L_i]\le O(\sqrt{b\log(k+1)})$ for each $i\in[k]$.
Here, expectations are over the fresh random choices within the block under this conditioning.
The same bounds hold for an incomplete block.
\end{lemma}

\begin{proof}
The fixed input and starting history determine the block endpoint independently of $Z$, $\xi$, and the learner's choices.
Abbreviate $C_A^{t_0}(t_0')$ as $C_A^{t_0}$.
Let $\nu$ be the number of switches back to the learner $A_{\mathrm{L}}$.
Suppose that $\ALG$ switches to $A_{\mathrm{R}}(Z)$ at time $t$ and back to $A_{\mathrm{L}}$ at time $t'>t$.
Then, $D(t)\ge\xi+1$ and $D(t')\le\xi$.
Using the definition of $D$ and the fact that $C_{A_{\mathrm{L}}}^{t_0}(t)$ is nondecreasing, we obtain
\begin{align*}
 C_{A_{\mathrm{R}}(Z)}^{t_0}(t')-C_{A_{\mathrm{R}}(Z)}^{t_0}(t)
 &=\bigl(C_{A_{\mathrm{L}}}^{t_0}(t')-D(t')\bigr)
   -\bigl(C_{A_{\mathrm{L}}}^{t_0}(t)-D(t)\bigr)
 =C_{A_{\mathrm{L}}}^{t_0}(t')-C_{A_{\mathrm{L}}}^{t_0}(t)+D(t)-D(t')\\
 &\ge D(t)-D(t')
 \ge(\xi+1)-\xi=1.
\end{align*}
Then, summing over all such intervals $[t,t']$, we have $C_{A_{\mathrm{R}}(Z)}^{t_0} \ge \nu$.
Since $Z\in\{2/3,1\}$ and $\OPT^{t_0}(t_0')=b$, Lemma~\ref{lem:random-comparison} yields $\nu\le C_{A_{\mathrm{R}}(Z)}^{t_0}\le5b/2$ for every realization.

We next show that $\nu$ is constant in expectation. Fix the realizations of $A_{\mathrm{L}}$ and $A_{\mathrm{R}}(Z)$, leaving only the randomness of $\xi$.
Let
$t_0=s_0<s_1<\cdots<s_N=t_0'$
be the sorted sequence of all distinct packet arrival times and acknowledgment times of $A_{\mathrm{L}}$ and $A_{\mathrm{R}}(Z)$ within the block, together with the block's start and end times.
These times are independent of $\xi$.
Split $[t_0,t_0']$ at $s_1^-,s_1,\ldots,s_{N-1}^-,s_{N-1},s_N^-$ into $2N$ subintervals $I_1,\ldots,I_{2N}$, where
$I_{2i-1}=[s_{i-1},s_i^-]$ and $I_{2i}=[s_i^-,s_i]$ for $i\in[N]$.
Here, $s_i^-$ denotes the instant immediately before the events at $s_i$, i.e. $I_{2i}$ represents the jump at $s_i$.
By construction, $D(t)$ is affine on each odd-indexed subinterval $I_{2i-1}$.
Write $I_i=[u_i,v_i]$ and let
$\mathcal{D}=\{i\in[2N]:D(u_i)>D(v_i)\}$
be the index set of the subintervals on which $D$ decreases.
Every switch from $A_{\mathrm{R}}(Z)$ back to $A_{\mathrm{L}}$ requires $D$ to cross $\xi$ downwards on some $I_i$ with $i\in\mathcal{D}$, and at most one such switch can occur on each of these subintervals.
Consequently, $\nu\le\sum_{i\in\mathcal{D}}\ind\{D(v_i)\le\xi\le D(u_i)\}$.
Taking expectations over $\xi$, which is uniform on $[0,b/8]$, we have
\begin{align*}
 \E_\xi[\nu]&\le\frac{8}{b}\sum_{i\in\mathcal{D}}\bigl(D(u_i)-D(v_i)\bigr).
\end{align*}
Moreover, since $C_{A_{\mathrm{L}}}^{t_0}(t)$ is nondecreasing, for every $i\in\mathcal{D}$ we have
\begin{align*}
 D(u_i)-D(v_i)
 &=\bigl(C_{A_{\mathrm{R}}(Z)}^{t_0}(v_i)-C_{A_{\mathrm{R}}(Z)}^{t_0}(u_i)\bigr)
   -\bigl(C_{A_{\mathrm{L}}}^{t_0}(v_i)-C_{A_{\mathrm{L}}}^{t_0}(u_i)\bigr)
 \le C_{A_{\mathrm{R}}(Z)}^{t_0}(v_i)-C_{A_{\mathrm{R}}(Z)}^{t_0}(u_i).
\end{align*}
Summing over $i\in\mathcal{D}$ and using the monotonicity of $C_{A_{\mathrm{R}}(Z)}^{t_0}(t)$, we obtain
\begin{align*}
 \E_\xi[\nu]
 \le \frac{8}{b} \sum_{i\in\mathcal{D}}\bigl(D(u_i)-D(v_i)\bigr)
 &\le \frac{8}{b} \sum_{i=1}^{2N}\bigl(C_{A_{\mathrm{R}}(Z)}^{t_0}(v_i)-C_{A_{\mathrm{R}}(Z)}^{t_0}(u_i)\bigr)
 = \frac{8}{b} C_{A_{\mathrm{R}}(Z)}^{t_0}.
\end{align*}
By Lemma~\ref{lem:random-comparison} with $x=2/3$, we also have $\E[C_{A_{\mathrm{R}}(Z)}^{t_0}]\le11b/6$.
Taking expectations also over $Z$ and the random choices of $A_{\mathrm{L}}$ therefore gives
$\E[\nu]\le\frac{8}{b}\E[C_{A_{\mathrm{R}}(Z)}^{t_0}] \le\frac{8}{b}\cdot\frac{11b}{6}=\frac{44}{3}$.

We next upper bound $X-C_{A_{\mathrm{R}}(Z)}^{t_0}$ and $X-C_{A_{\mathrm{L}}}^{t_0}$ in terms of $\nu$ and $\xi$. Let
$t_0=r_0<r_1<\cdots<r_M=t_0'$,
where $r_1,\ldots,r_{M-1}$ are the switching times within the block, i.e.,
on the interval $(r_{j-1},r_j]$, $\ALG$ follows $A_{\mathrm{L}}$ if $j$ is odd and $A_{\mathrm{R}}(Z)$ if $j$ is even.
Since $\nu$ counts the switches back to $A_{\mathrm{L}}$, there are $\nu+1$ odd-indexed intervals and at most $\nu+1$ even-indexed intervals; in particular, $M\le2\nu+2$.
For $j\in[M]$ and $A\in\{A_{\mathrm{L}},A_{\mathrm{R}}(Z)\}$, write $\Delta_j C_A=C_A^{t_0}(r_j)-C_A^{t_0}(r_{j-1})$.

Each of the $M-1$ switches and the block endpoint requires at most one additional acknowledgment. Thus,
\begin{align*}
 X
 &\le\sum_{\substack{j\in[M]\\j\text{ odd}}}\Delta_j C_{A_{\mathrm{L}}}
    +\sum_{\substack{j\in[M]\\j\text{ even}}}\Delta_j C_{A_{\mathrm{R}}(Z)}
    +(M-1)+1
 \le\sum_{\substack{j\in[M]\\j\text{ odd}}}\Delta_j C_{A_{\mathrm{L}}}
    +\sum_{\substack{j\in[M]\\j\text{ even}}}\Delta_j C_{A_{\mathrm{R}}(Z)}
    +2\nu+2.
\end{align*}
For either simulation $A$, we have $C_A^{t_0}=\sum_{j=1}^M\Delta_j C_A$.
By the definition of $D$, this gives
\begin{align*}
 X-C_{A_{\mathrm{R}}(Z)}^{t_0}
 &\le\sum_{\substack{j\in[M]\\j\text{ odd}}}
       \bigl(\Delta_j C_{A_{\mathrm{L}}}-\Delta_j C_{A_{\mathrm{R}}(Z)}\bigr)+2\nu+2
 =\sum_{\substack{j\in[M]\\j\text{ odd}}}
       \bigl(D(r_j)-D(r_{j-1})\bigr)+2\nu+2,\\
 X-C_{A_{\mathrm{L}}}^{t_0}
 &\le\sum_{\substack{j\in[M]\\j\text{ even}}}
       \bigl(\Delta_j C_{A_{\mathrm{R}}(Z)}-\Delta_j C_{A_{\mathrm{L}}}\bigr)+2\nu+2
 =\sum_{\substack{j\in[M]\\j\text{ even}}}
       \bigl(D(r_{j-1})-D(r_j)\bigr)+2\nu+2.
\end{align*}

It remains to bound the increments of $D$ appearing in these sums.
Each jump of $D(t)$ has size at most one, since waiting costs are continuous and each simulation pays one per acknowledgment.
The switching thresholds $\xi+1$ and $\xi$ therefore imply that $D(r_j)\le\xi+2$ for odd $j\in[M]$ and $D(r_j)\ge\xi-1$ for even $j\in[M]$.
Since $D(r_0)=0$, we obtain
\begin{align*}
 D(r_1)-D(r_0)&\le\xi+2,\\
 D(r_j)-D(r_{j-1})&\le(\xi+2)-(\xi-1)=3
 &&(3\le j\le M,\ j\text{ odd}),\\
 D(r_{j-1})-D(r_j)&\le(\xi+2)-(\xi-1)=3
 &&(j\in[M],\ j\text{ even}).
\end{align*}
There are $\nu$ odd-indexed intervals after the initial interval and at most $\nu+1$ even-indexed intervals.
Substituting these increment bounds into the two sums above yields
\begin{align*}
 X-C_{A_{\mathrm{R}}(Z)}^{t_0}
 &\le(\xi+2)+3\nu+(2\nu+2)=\xi+5\nu+4,\\
 X-C_{A_{\mathrm{L}}}^{t_0}
 &\le3(\nu+1)+(2\nu+2)=5\nu+5.
\end{align*}

Using the bounds above and $\xi\le b/8$, we obtain
$X\le C_{A_{\mathrm{R}}(Z)}^{t_0}+\xi+5\nu+4 \le15b+\frac{b}{8}+4\le16b$.
Averaging over $\xi$, $Z$, and the learner's choices, with $\E[\xi]=b/16$, gives
\begin{align*}
 \E[X]&\le\E[C_{A_{\mathrm{R}}(Z)}^{t_0}]+\E[\xi]+5\E[\nu]+4
 \le\frac{11b}{6}+\frac{b}{16}+\frac{220}{3}+4
 =\frac{91b}{48}+\frac{232}{3}\le\left(2-\frac1{36}\right)b,
\end{align*}
using $b\ge\theta=2^{10}$.
Using $\E[\nu]\le44/3$, we obtain
$\E[X-C_{A_{\mathrm{L}}}^{t_0}]\le5\E[\nu]+5 \le\frac{220}{3}+5=O(1)$.

For an expert $A_i$ with $L_i>2b$, the regret bound follows from $\E[X]<2b$.
For $L_i\le2b$, apply the proof of Lemma~\ref{lem:expected-switching-learner} to the expert cost increments within the block.
Combining this with the preceding estimate gives $\E[X-L_i]\le O(\eta b+\log k/\eta+1)$.
For $b\ge4\ln(k+1)$, substituting $\eta=\sqrt{\ln(k+1)/b}$ gives $\E[X-L_i]\le O(\sqrt{b\log(k+1)})$.
For $b<4\ln(k+1)$, the bound $X\le16b$ gives $\E[X-L_i]\le16b=O(\sqrt{b\log(k+1)})$.

For an incomplete block, append sufficiently late dummy packets to complete it and apply the same arguments to the portion preceding the dummy arrivals.
\end{proof}

We are now ready to prove Theorem~\ref{thm:simultaneous-guarantees}.

\begin{proofof}{Theorem~\ref{thm:simultaneous-guarantees}}
We first establish the competitive guarantee, then bound the expected regret.

\paragraph{Competitive guarantee.}
We show that the budget $\Phi(t)$ remains nonnegative.
Algorithm~\ref{alg:recovery} keeps the budget at least its starting value and leaves no pending packets when it returns.
By Lemma~\ref{lem:learning-block}, the cost incurred by Algorithm~\ref{alg:competitive} during a block starting at time $t_0$ is at most $16b=16\Phi(t_0)/\theta\le\Phi(t_0)/2$.
Thus, $C_{\ALG}(t)-C_{\ALG}(t_0)\le\Phi(t_0)/2$ throughout the block.
By the definition $\Phi(t)=(2+\varepsilon)\OPT(t)-C_{\ALG}(t)$ and the monotonicity of $\OPT(t)$, we have
\begin{align*}
 \Phi(t)&=\Phi(t_0)+(2+\varepsilon)\bigl(\OPT(t)-\OPT(t_0)\bigr)-\bigl(C_{\ALG}(t)-C_{\ALG}(t_0)\bigr)
 \ge\Phi(t_0)-\frac{\Phi(t_0)}2=\frac{\Phi(t_0)}2.
\end{align*}
This proves the competitive guarantee.

\paragraph{Regret bound.}
We sum the contributions from the blocks and Algorithm~\ref{alg:recovery}.
Fix the input and the realizations of the given experts.
Before the first block, line~\ref{line:increase-budget} of Algorithm~\ref{alg:competitive} repeatedly starts Algorithm~\ref{alg:recovery} until $\Phi(t)\ge\theta^2$.
Each execution that returns increases the budget by at least one by Lemma~\ref{lem:recovery}.
Starting from budget zero, this requires at most $\theta^2$ executions, including a possible final one that never returns to Algorithm~\ref{alg:competitive}.
Lemma~\ref{lem:recovery} also bounds the total expected cost before the first block by $O(\log(1/\varepsilon))$, since $\theta$ is a fixed constant.
If the first block starts, the preceding budget bound gives $\theta^2/2\le\Phi(t)\le(2+\varepsilon)\OPT$ at all subsequent times.

We use the expected increase in $\sqrt{\Phi}$ to sum the block regret bounds, and the expected decrease in $1/\Phi$ to bound how often Algorithm~\ref{alg:recovery} is run.
Consider a completed block from time $t_0$ to $t_0'$ with size $b=\Phi(t_0)/\theta$ and cost $X$, and condition on its starting history.
By \eqref{eq:partial-superadditive}, $\Phi(t_0')\ge\Phi(t_0)+(2+\varepsilon)b-X=\Phi(t_0)(1+u)$, where $u=(2+\varepsilon-X/b)/\theta$.
By Lemma~\ref{lem:learning-block}, $\E[X]\le(2-1/36)b$ and $0\le X\le16b$.
Thus, $\E[u]\ge(\varepsilon+1/36)/\theta$ and $|u|\le14/\theta<1/64$.
To bound the second moment, we use $X^2\le16bX$ and obtain
\begin{align*}
 \E[u^2]&\le\frac{(2+\varepsilon)^2+(12-2\varepsilon)\E[X]/b}{\theta^2}
 \le\frac{(2+\varepsilon)^2+2(12-2\varepsilon)}{\theta^2}
 =\frac{28+\varepsilon^2}{\theta^2}.
\end{align*}
Since $|u|\le1/64$, we have $\sqrt{1+u}\ge1+u/2-u^2/4$ and $(1+u)^{-1}=1-u+u^2/(1+u)\le1-u+(64/63)u^2$.
These inequalities give
\begin{align}
 \E\left[\sqrt{\Phi(t_0')}-\sqrt{\Phi(t_0)}\right]
 &\ge\sqrt{\Phi(t_0)}\left(\frac{1}{2}\cdot\frac{\varepsilon+1/36}{\theta}-\frac{1}{4}\cdot\frac{28+\varepsilon^2}{\theta^2}\right)
 \ge\frac{\sqrt{\Phi(t_0)}}{256\theta},\label{eq:budget-sqrt}\\
 \E\left[\frac{1}{\Phi(t_0')}\right]
 &\le\frac{1}{\Phi(t_0)}\left(1-\frac{\varepsilon+1/36}{\theta}+\frac{64}{63}\cdot\frac{28+\varepsilon^2}{\theta^2}\right)
 \le\frac{1}{\Phi(t_0)},\label{eq:budget-reciprocal}
\end{align}
where we used $\theta=2^{10}$ and $0<\varepsilon<1$.

We next bound the sum of the square roots of the block sizes.
Let $h$ be the number of completed blocks, and let the $j$th completed block start at $t_j$, end at $t_j'$, and have size $b_j=\Phi(t_j)/\theta$.
Lemma~\ref{lem:recovery} implies $\Phi(t_{j+1})\ge\Phi(t_j')$ for $j<h$.
Thus, for $h\ge1$,
\begin{align*}
 \sum_{j=1}^{h}\bigl(\sqrt{\Phi(t_j')}-\sqrt{\Phi(t_j)}\bigr)
 &=\sqrt{\Phi(t_h')}-\sqrt{\Phi(t_1)}
   -\sum_{j=1}^{h-1}\bigl(\sqrt{\Phi(t_{j+1})}-\sqrt{\Phi(t_j')}\bigr)\\
 &\le\sqrt{\Phi(t_h')}-\sqrt{\Phi(t_1)}
 \le\sqrt{(2+\varepsilon)\OPT}.
\end{align*}
Since $h\theta\le\sum_{j=1}^h b_j\le\OPT$ by \eqref{eq:partial-superadditive}, the number of completed blocks is bounded for the fixed input.
Whether a block completes is determined by its starting history.
Thus, taking expectations over the starting histories and summing \eqref{eq:budget-sqrt} yields
\begin{align}
 \E\left[\sum_{j=1}^{h}\sqrt{b_j}\right]
 &=\frac1{\sqrt{\theta}}\,\E\left[\sum_{j=1}^{h}\sqrt{\Phi(t_j)}\right]
 \le256\sqrt{\theta}\,\E\left[\sum_{j=1}^{h}\bigl(\sqrt{\Phi(t_j')}-\sqrt{\Phi(t_j)}\bigr)\right]\notag\\
 &\le256\sqrt{\theta(2+\varepsilon)\OPT}=O(\sqrt{\OPT}),
 \label{eq:block-size-sum}
\end{align}
where the first equality uses $b_j=\Phi(t_j)/\theta$, the first inequality follows from \eqref{eq:budget-sqrt}, and the second inequality uses the preceding displayed bound.

Let $R$ be the total number of executions of Algorithm~\ref{alg:recovery} started over the entire input; we next show that $\E[R]=O(1)$.
For $h\ge1$, let $t_{h+1}$ be the time of the last return from Algorithm~\ref{alg:recovery} after $t_h'$, or $t_h'$ if there is no such return.
For each $j\in[h]$, let $R_j$ be the number of executions of Algorithm~\ref{alg:recovery} that return during $(t_j',t_{j+1}]$.
Since there are at most $\theta^2$ executions before the first block and at most one execution after it that never returns, we have $R\le\theta^2+\sum_{j=1}^h R_j+1$.
It remains to bound $\E[\sum_{j=1}^h R_j]$.
Each execution counted by $R_j$ starts with a budget $p\in(0,\theta^2)$ and returns with a budget at least $p+1$ by Lemma~\ref{lem:recovery}.
It therefore decreases $1/\Phi$ by at least
$\frac1p-\frac1{p+1} =\frac1{p(p+1)}\ge\frac1{\theta^2(\theta^2+1)}$.
By adding these decreases over the intervals $(t_j',t_{j+1}]$, for $h\ge1$, we have
\begin{align*}
 \sum_{j=1}^{h}\frac{R_j}{\theta^2(\theta^2+1)}
 &\le\sum_{j=1}^{h}\left(\frac1{\Phi(t_j')}-\frac1{\Phi(t_{j+1})}\right)
 =\frac1{\Phi(t_1)}-\frac1{\Phi(t_{h+1})}
   +\sum_{j=1}^{h}\left(\frac1{\Phi(t_j')}-\frac1{\Phi(t_j)}\right)\\
 &\le\frac1{\theta^2}
   +\sum_{j=1}^{h}\left(\frac1{\Phi(t_j')}-\frac1{\Phi(t_j)}\right),
\end{align*}
where the last inequality uses $\Phi(t_1)\ge\theta^2$ and $1/\Phi(t_{h+1})\ge0$.
Taking expectations and applying \eqref{eq:budget-reciprocal} to each completed block yields
\begin{align*}
 \frac{\E[\sum_{j=1}^h R_j]}{\theta^2(\theta^2+1)}
 &\le\frac1{\theta^2}
   +\E\left[\sum_{j=1}^{h}\left(\frac1{\Phi(t_j')}-\frac1{\Phi(t_j)}\right)\right]
 \le\frac1{\theta^2}.
\end{align*}
Thus, $\E[\sum_{j=1}^h R_j]\le\theta^2+1$, and hence $\E[R]\le2\theta^2+2=O(1)$.

Lemma~\ref{lem:recovery} bounds the expected cost of each of the $R$ executions, conditional on its starting history, by $O(\log(1/\varepsilon))$.
Taking expectations over these histories and summing over the executions gives the total expected cost of these $R$ executions as $O\left(\left(1+\log(1/\varepsilon)\right)\E[R]\right)=O\left(1+\log(1/\varepsilon)\right)$.

Choose an expert $A_{i^*}$ with minimum total cost under the fixed realizations.
For each completed block $j\in[h]$, let $X_j$ be the cost incurred by Algorithm~\ref{alg:competitive} and $L_{i^*,j}$ the cost incurred by $A_{i^*}$ during that block.
Summing the bounds from Lemma~\ref{lem:learning-block} gives
\begin{align*}
 \E\left[\sum_{j=1}^{h}(X_j-L_{i^*,j})\right]
 &\le O\left(\E\left[\sum_{j=1}^{h}\sqrt{b_j\log(k+1)}\right]\right)
 =O\left(\sqrt{\log(k+1)}\cdot\E\left[\sum_{j=1}^{h}\sqrt{b_j}\right]\right)\\
 &=O\bigl(\sqrt{\OPT\log(k+1)}\bigr),
\end{align*}
where the last equality uses \eqref{eq:block-size-sum}.
If there is a final incomplete block, its size is at most $(2+\varepsilon)\OPT/\theta$, and Lemma~\ref{lem:learning-block} bounds its expected regret by $O\bigl(\sqrt{\OPT\log(k+1)}\bigr)$ as well.
Adding the costs of Algorithm~\ref{alg:recovery} therefore yields $\E[C_{\ALG}-C_{A_{i^*}}]\le O\bigl(\sqrt{\OPT\log(k+1)}+1+\log(1/\varepsilon)\bigr)$.
For a nonempty input, $\OPT\ge1$ absorbs the constant into the square-root term, and an empty input has zero cost.
Since $C_{A_{i^*}}=\min_{i\in[k]}C_{A_i}$ under the fixed realizations, averaging over the expert realizations proves the regret bound.
\end{proofof}

Finally, we remark that Algorithm~\ref{alg:competitive} uses a polynomial number of arithmetic operations with respect to $k$ and $n$. 
This is because $\OPT(t)$ can be computed by dynamic programming and the number of events (e.g., packet arrivals, switches, block endpoints, rounds, and acknowledgments of experts) requiring explicit processing is bounded by a polynomial of $k$ and $n$.
In particular, when $\ALG$ has no pending packets, computations can be deferred until the next packet arrival, at which point only the necessary part of the execution is computed.

\section{Lower bounds}
\label{sec:lower-bounds}

We establish two lower bounds that together match the dependence on $\OPT$, $k$, and $\varepsilon$ in Theorem~\ref{thm:simultaneous-guarantees} for $k\ge2$.
The first quantifies the regret required by a competitive ratio close to two, while the second establishes the dependence on the number of experts without imposing a competitive requirement.

\subsection{Dependence on \texorpdfstring{$\varepsilon$}{epsilon}}

Let $a_i=3(i-1)^2$ be the arrival time of packet $i$ for $i\ge1$, and let $\sigma_n$ consist of the first $n$ packets.
On these inputs, a competitive ratio close to two forces substantial waiting on the initial packets.
Under the same guarantee with competitive ratio two, the algorithm incurs regret $\Omega(\OPT)$.

\begin{theorem}
\label{thm:exact-lower}
For all $0< \varepsilon\le 1/16$, every randomized online algorithm $\ALG$ satisfying $C_{\ALG}\le(2+\varepsilon)\OPT$ on every input and every realization satisfies, for every $n\ge1$ and every realization,
\begin{align}
 C_{\ALG}(\sigma_n)-\OPT(\sigma_n)
 &\ge\Omega\left(\min\{\OPT(\sigma_n),\,\log(1/\varepsilon)\}\right).
 \label{eq:epsilon-lower}
\end{align}
For $\varepsilon=0$, the same competitive requirement forces regret $\Omega(\OPT(\sigma_n))$ for every $n\ge1$ and every realization.
\end{theorem}

\begin{proof}
Since serving multiple packets with one acknowledgment adds at least three units of waiting cost per acknowledgment saved, $\OPT(\sigma_n)=n$.
For $0\le\varepsilon<1$, the guarantee on the prefix $\sigma_i$ gives $C_{\ALG}(\sigma_i)\le(2+\varepsilon)i<3i\le a_{i+1}-a_i$ for every realization of the random choices, requiring packet $i$ to be acknowledged before $a_{i+1}$.
Thus, each acknowledgment serves exactly one packet for every realization.
For $i\in[n]$, let $\Delta_i$ be the random variable giving the time from the arrival of packet $i$ until $\ALG$ acknowledges it on the fixed input $\sigma_n$.
The regret of $\ALG$ is therefore $C_{\ALG}(\sigma_n)-n=\sum_{i=1}^n\Delta_i$ for every realization.

We show by contradiction that $\Delta_i\ge1-\varepsilon(2^{i+1}-1)$ for every $i\in[n]$ and every realization.
Suppose instead that there is a realization of $\ALG$'s internal randomness for which $\Delta_i<1-\varepsilon(2^{i+1}-1)$ for some $i\in[n]$.
Fix such a realization, and let $j$ be the first index violating the bound.
Throughout the following contradiction argument, $\Delta_i$ denotes its value under this realization, and $\ALG$ is deterministic and retains the competitive guarantee on every input.
Starting with $\sigma_j$, we add packets to construct an input that violates the competitive guarantee for the fixed realization.

First, the choice of $j$ implies
\begin{align*}&
 (2+\varepsilon)(j+\Delta_j)-\left(j+\sum_{i=1}^j\Delta_i\right)
 =(1+\varepsilon)(j+\Delta_j)-\sum_{i=1}^{j-1}\Delta_i\\&
 <(1+\varepsilon)\bigl(j+1-\varepsilon(2^{j+1}-1)\bigr)-j+1+\varepsilon(2^{j+1}-j-3)
 =2-\varepsilon-\varepsilon^2(2^{j+1}-1)\le2-\varepsilon.
\end{align*}
Since $2-\varepsilon\le1+1/(1+\varepsilon)$, choose $c\in[0,1/(1+\varepsilon))$ such that the left-hand side is strictly less than $1+c$.
We keep this $c$ fixed throughout the construction.

At the $\ell$th added arrival, release $M_{\ell}$ packets simultaneously at time $t_{\ell}$.
The first such arrival is at $t_1=a_j+\Delta_j+1/M_1^2$; for $\ell\ge1$, the next arrival is $1/M_{\ell+1}^2$ time units after $\ALG$ acknowledges the $M_{\ell}$ packets.
Let $P_\ell=1+\sum_{\ell'=1}^{\ell-1}M_{\ell'}$ count packet $j$ and all packets added before $t_{\ell}$.
Choose the integers $M_{\ell}\ge2$ successively so that
\begin{align*}
 (2+\varepsilon)\frac{P_{\ell}c}{M_{\ell}}\le\frac{1-(1+\varepsilon)c}{4}
\qquad\text{and}\qquad
(2+\varepsilon)\frac{P_{\ell}}{M_{\ell}^2}\le\frac{1-(1+\varepsilon)c}{4}.
\end{align*}
These conditions are possible because $P_{\ell}$ is determined by the earlier choices and $1-(1+\varepsilon)c>0$.
For $M_1$, also require
\begin{align*}
 (2+\varepsilon)(j+\Delta_j+1/M_1^2)
 -\left(j+\sum_{i=1}^j\Delta_i\right)\le1+c,
\end{align*}
which is possible by the strict inequality used to choose $c$.

For each $\ell$, consider the input consisting of $\sigma_j$ and the first $\ell$ added arrivals, with no packets arriving after $t_{\ell}$.
Consider an offline schedule that acknowledges packets $1,\ldots,j-1$ immediately upon arrival, and acknowledges packet $j$ and all the added packets together at $t_{\ell}$.
Let $S_{\ell}$ be the cost of this offline schedule; since it is feasible, $\OPT\le S_{\ell}$ for this input.
At time $t_{\ell}$, before any acknowledgment at that time, consider the difference between $(2+\varepsilon)S_{\ell}$ and the cost $\ALG$ has incurred so far.
Since $S_1=j+\Delta_j+1/M_1^2$, the choice of $M_1$ makes this difference at most $1+c$ at the first added arrival.
If the input ended there, competitiveness would require $\ALG$ to serve the $M_1$ packets within this remaining allowance.
Since acknowledgment costs one, their total waiting cost is at most $c$.

We now show that the same bound holds at each added arrival, while the difference decreases by a fixed positive amount.
Suppose the $M_{\ell}$ packets have total waiting cost $w_{\ell}\le c$.
Then, $\ALG$ acknowledges them $w_{\ell}/M_{\ell}$ time units after $t_{\ell}$. Hence, $t_{\ell+1}-t_{\ell}=\frac{w_{\ell}}{M_{\ell}}+\frac{1}{M_{\ell+1}^2}$.
Moving this offline schedule's final acknowledgment from $t_{\ell}$ to $t_{\ell+1}$ makes all $P_{\ell}+M_{\ell}$ packets already present wait this much longer, while the newly arriving packets incur no waiting cost.
Thus,
\begin{align*}
 S_{\ell+1}-S_{\ell}
 &=(P_{\ell}+M_{\ell})\left(\frac{w_{\ell}}{M_{\ell}}+\frac{1}{M_{\ell+1}^2}\right)
 =w_{\ell}+\frac{P_{\ell}w_{\ell}}{M_{\ell}}+\frac{P_{\ell}+M_{\ell}}{M_{\ell+1}^2}.
\end{align*}
The first term is the waiting cost of the $M_{\ell}$ packets under $\ALG$.
The second accounts for the earlier packets waiting until $\ALG$ acknowledges those packets, and the third for the short delay until the next arrival.
Since $P_{\ell+1}=P_{\ell}+M_{\ell}$ and $w_{\ell}\le c$, the choices of $M_{\ell}$ and $M_{\ell+1}$ bound $(2+\varepsilon)$ times the last two terms by $[1-(1+\varepsilon)c]/2$.
Meanwhile, $\ALG$ pays $1+w_\ell$ between the two arrivals.
Consequently,
\begin{align*}
 1+w_{\ell}-(2+\varepsilon)(S_{\ell+1}-S_{\ell})
 &\ge1-(1+\varepsilon)w_{\ell}-\frac{1-(1+\varepsilon)c}{2}
 \ge\frac{1-(1+\varepsilon)c}{2}>0.
\end{align*}
Thus, at $t_{\ell+1}$, the remaining cost allowance $(2+\varepsilon)S_{\ell+1}$ minus the cost already incurred by $\ALG$ is still at most $1+c$.
Applying the competitive guarantee to the input ending at $t_{\ell+1}$ again bounds the total waiting cost of the newly arriving packets by $c$, completing the induction.
After finitely many added arrivals, the difference becomes negative: $\ALG$ has already paid more than $(2+\varepsilon)S_{\ell}\ge(2+\varepsilon)\OPT$ for the input ending at that arrival.
This contradicts the competitive guarantee and proves the bound on $\Delta_i$ for every realization on $\sigma_n$.

For $0<\varepsilon\le1/16$, take $m=\min\{n,\lfloor\log_2(1/(4\varepsilon))\rfloor\}$.
Then $\Delta_i\ge1/2$ for every $i\le m$, giving regret at least $m/2=\Omega(\min\{n,\log(1/\varepsilon)\})$.
For $\varepsilon=0$, every $\Delta_i\ge1$, and the competitive guarantee forces $C_{\ALG}(\sigma_n)=2n$.
Thus, the regret is $n=\OPT(\sigma_n)$, giving the claimed $\Omega(\OPT(\sigma_n))$ bound.
\end{proof}

\subsection{Dependence on the number of experts}
\label{sec:regret-lower-bound}

To establish a regret lower bound in terms of the number of experts, we construct a family of $\Theta(k)$ experts.

\begin{theorem}
\label{thm:sqrt-regret-lower}
For every $k\ge2$, there is a family $\mathcal{E}_k$ of at most $k$ deterministic online experts containing $A_{\mathrm{R}}(0)$ such that, for every randomized online algorithm $\ALG$ and every integer $h\ge1$, some fixed input with $h\le\OPT\le3h/2$ satisfies
\begin{align}
 \E[C_{\ALG}]-\min_{A\in\mathcal{E}_k}C_A
 &\ge\Omega\left(\min\{\OPT,\sqrt{\OPT\log k}\}\right).
 \label{eq:experts-lower}
\end{align}
\end{theorem}

\begin{proof}
Set $g=\lfloor\log_2 k\rfloor$ and let $t_j=3(j-1)$ for $j\ge1$.
Draw independent Bernoulli random variables $Y_1,\ldots,Y_h$ with parameter $1/2$.
For each $j\in[h]$, release packets at $t_j$ and $t_j+1/2$ if $Y_j=1$, and only at $t_j$ if $Y_j=0$.

Assign each interval $[t_j,t_j+3)$ to group $i\in[g]$ with $j\equiv i\pmod g$.
Each expert fixes one of two policies for each group and applies it in every interval of that group.
The two policies acknowledge each packet immediately or acknowledge all pending packets at $t_j+1/2$ in interval $[t_j,t_j+3)$.
Include all $2^g$ combinations to obtain a family $\mathcal{E}_k$ of $2^g\le k$ experts.
This family is independent of $h$, its experts are online on every input, and choosing immediate acknowledgment in every group gives $A_{\mathrm{R}}(0)$.

By Yao's principle, it suffices to prove the regret lower bound in expectation over this input distribution for every deterministic online algorithm.
Fix such an algorithm $\ALG$.
Its cost within $[t_j,t_j+3)$ is at least $1+Y_j$ if it acknowledges the first packet strictly before $t_j+1/2$, and at least $3/2$ otherwise.
These bounds also hold if a packet remains unacknowledged until $t_j+3$ or later, since it arrives no later than $t_j+1/2$ and thus incurs at least $5/2$ waiting cost within the interval.
Summing over disjoint intervals therefore gives $\E_Y[C_{\ALG}]\ge3h/2$.

An offline optimum never serves packets from different intervals with the same acknowledgment.
Inserting an acknowledgment at the last arrival of the earliest of these intervals saves at least $5/2$ waiting cost at a cost of one.
It therefore pays one for a single-packet interval and $3/2$ for a two-packet interval.
Hence, $h\le\OPT\le3h/2$.

For $i\in[g]$, let $h_i$ count the intervals in group $i$ in which packets arrive, and let $X_i$ be the sum of $2Y_j-1$ over these intervals.
Thus, $X_i$ is the number of two-packet intervals minus the number of single-packet intervals in group $i$.
These two types of intervals number $(h_i+X_i)/2$ and $(h_i-X_i)/2$, respectively.
Immediate acknowledgment costs two on each two-packet interval and one on each single-packet interval, giving a total cost of
$2\cdot\frac{h_i+X_i}{2}+\frac{h_i-X_i}{2} =\frac{3h_i}{2}+\frac{X_i}{2}$.
The policy that acknowledges all pending packets at $t_j+1/2$ in each interval $[t_j,t_j+3)$ incurs one acknowledgment and a waiting cost of $1/2$ for the first packet; the second packet, if present, is acknowledged upon arrival and incurs no waiting cost.
Its total cost on group $i$ is therefore $(1+1/2)h_i=3h_i/2$.
Since the experts include every combination of the two policies across groups, the best expert can choose the cheaper policy separately for each group.
Using $\sum_{i=1}^g h_i=h$ and $\min\{x,0\}=(x-|x|)/2$, we obtain, for every realization of the input,
\begin{align*}
 \min_{A\in\mathcal{E}_k}C_A
 &=\sum_{i=1}^g\min\left\{\frac{3h_i}{2}+\frac{X_i}{2},\frac{3h_i}{2}\right\}
 =\frac{3h}{2}+\frac12\sum_{i=1}^g\min\{X_i,0\}
 =\frac{3h}{2}+\frac14\sum_{i=1}^g\bigl(X_i-|X_i|\bigr).
\end{align*}
Since $\E_Y[X_i]=0$, taking expectations gives
\begin{align*}
 \E_Y\left[\min_{A\in\mathcal{E}_k}C_A\right]
 &=\frac{3h}{2}-\frac14\sum_{i=1}^g\E_Y[|X_i|].
\end{align*}

To bound $\E_Y[|X_i|]$ from below, fix a group with $h_i>0$ and write $X_i=\sum_{r=1}^{h_i}Z_r$, where the $Z_r$ are independent and take values $1$ and $-1$ with equal probability.
Thus, $\E_Y[Z_r]=0$ and $Z_r^2=1$, giving
\begin{align*}
 \E_Y[X_i^2]
 &=\sum_{r=1}^{h_i}\E_Y[Z_r^2]
   +2\sum_{1\le r<s\le h_i}\E_Y[Z_rZ_s]
 =h_i+2\sum_{1\le r<s\le h_i}\E_Y[Z_r]\E_Y[Z_s]
 =h_i.
\end{align*}
In the expansion of $X_i^4$, independence and symmetry imply that any term in which some $Z_r$ has an odd exponent has expectation zero.
The surviving terms have either one index occurring four times or two distinct indices occurring twice each; for each pair of distinct indices, there are $\binom{4}{2}=6$ arrangements.
Consequently,
\begin{align*}
 \E_Y[X_i^4]
 &=\sum_{r=1}^{h_i}\E_Y[Z_r^4]
   +6\sum_{1\le r<s\le h_i}\E_Y[Z_r^2Z_s^2]
 =h_i+6\binom{h_i}{2}
 =3h_i^2-2h_i.
\end{align*}
Next, apply H\"older's inequality with conjugate exponents $3/2$ and $3$ to $|X_i|^{2/3}$ and $|X_i|^{4/3}$:
\begin{align*}
 \E_Y[X_i^2]
 &=\E_Y\bigl[|X_i|^{2/3}|X_i|^{4/3}\bigr]
 \le\left(\E_Y\bigl[(|X_i|^{2/3})^{3/2}\bigr]\right)^{2/3}
       \left(\E_Y\bigl[(|X_i|^{4/3})^3\bigr]\right)^{1/3}\\
 &=\bigl(\E_Y[|X_i|]\bigr)^{2/3}
   \bigl(\E_Y[X_i^4]\bigr)^{1/3}.
\end{align*}
Raising both sides to the power $3/2$ and rearranging, then substituting the moment formulas above, yields
\begin{align*}
 \E_Y[|X_i|]
 &\ge\frac{(\E_Y[X_i^2])^{3/2}}{(\E_Y[X_i^4])^{1/2}}
 =\frac{h_i^{3/2}}{\sqrt{3h_i^2-2h_i}}
 \ge\frac{h_i^{3/2}}{\sqrt{3h_i^2}}
 =\sqrt{h_i/3}.
\end{align*}
This bound also holds when $h_i=0$, since then $X_i=0$.
If $h\ge g$, then $h_i\ge h/(2g)$ for every $i\in[g]$, and otherwise exactly $h$ of the $h_i$ equal one and the rest are zero.
Hence, $\sum_{i=1}^g\sqrt{h_i}\ge\min\{h,\sqrt{hg}\}/\sqrt2$.
Since $g=\Theta(\log k)$, these estimates give expected regret $\Omega(\min\{h,\sqrt{h\log k}\})$ for every deterministic $\ALG$.
Yao's principle yields the same lower bound on some fixed input for every randomized algorithm.
The claim follows from $h\le\OPT\le3h/2$.
\end{proof}

Fix $k\ge2$, an integer $h\ge1$, and a randomized online algorithm $\ALG$ satisfying the realized $(2+\varepsilon)$ competitive guarantee.
Use the expert family $\mathcal{E}_k$ from Theorem~\ref{thm:sqrt-regret-lower}.
Since $A_{\mathrm{R}}(0)\in\mathcal{E}_k$ is optimal on $\sigma_h$, the best expert's cost on this input equals $\OPT(\sigma_h)=h$.
Thus, for $0<\varepsilon\le1/16$, Theorem~\ref{thm:exact-lower} gives regret $\Omega(\min\{h,\log(1/\varepsilon)\})$ against $\mathcal{E}_k$ on $\sigma_h$.
Theorem~\ref{thm:sqrt-regret-lower} gives regret $\Omega(\min\{h,\sqrt{h\log k}\})$ against the same family on another fixed input with $h\le\OPT\le3h/2$.
Choose whichever of these two inputs gives the larger lower bound.
Using
$\max\{\min(h,a),\min(h,b)\}\ge\tfrac12\min(h,a+b)$
with $a=\sqrt{h\log k}$ and $b=\log(1/\varepsilon)$ gives regret $\Omega(\min\{h,\sqrt{h\log k}+\log(1/\varepsilon)\})$.
For $1/16<\varepsilon<1$, the same bound follows from Theorem~\ref{thm:sqrt-regret-lower} alone, since $\log(1/\varepsilon)$ is bounded and $\sqrt{h\log k}$ is bounded below by a positive constant.
Since $h\le\OPT\le3h/2$ on either input, we conclude that, for every $\varepsilon\in(0,1)$, some fixed input in this range satisfies
\begin{align}
 \E[C_{\ALG}]-\min_{A\in\mathcal{E}_k}C_A
 &\ge\Omega\left(\min\left\{\OPT,\sqrt{\OPT\log k}+\log\frac1\varepsilon\right\}\right)
 \label{eq:joint-lower}
\end{align}
This lower bound matches the upper bound in Theorem~\ref{thm:simultaneous-guarantees} up to constant factors.

The same distribution rules out sublinear expected regret against the offline optimum.
For every randomized online algorithm and every $h\ge1$, we have $\E_Y[\OPT]=5h/4$ and $\E_Y\E[C_{\ALG}]\ge3h/2$.
Averaging yields a fixed input with $h\le\OPT\le3h/2$ and $\E[C_{\ALG}]-\OPT\ge h/4\ge\OPT/6$.



\section*{Acknowledgments}

This work was partially supported by the joint project of Kyoto University and Toyota Motor Corporation, titled ``Advanced Mathematical Science for Mobility Society'', 
JST ERATO Grant Number JPMJER2301,
JST ASPIRE Grant Number JPMJAP2302,
and
JSPS KAKENHI Grant Numbers 
JP21K17708, 
JP21H03397, 
JP25K00137, 
JP26K14718,
JP26K02867, and
JP26KJ1106. 

The authors developed the results and proofs with assistance from OpenAI ChatGPT/Codex and Anthropic Claude.
These tools also assisted with language editing and literature discovery.
The authors take full responsibility for the final manuscript.

\bibliography{references}

\appendix
\section{Omitted Proofs}
\label{app:preliminary-proofs}

\subsection{Proof of Lemma~\ref{lem:two-thresholds}}
\label{app:proof-two-thresholds}

Each completed acknowledgment interval of $A_{\mathrm{R}}(1)$ contains an acknowledgment of $A_{\mathrm{R}}(z)$, since its triggering packets would otherwise force $A_{\mathrm{R}}(z)$ to acknowledge by the endpoint.
This proves $N_1(t)\le N_z(t)$.
The triggering packets in each completed interval of $A_{\mathrm{R}}(z)$ contribute at least $z$ waiting cost, proving the first displayed inequality.

For the second, compare with any schedule having $m$ acknowledgments and waiting cost $w$ through $t$.
Waiting paid by $A_{\mathrm{R}}(z)$ after this schedule has served a packet can be charged to that schedule's acknowledgment.
If $q\ge1$ packets remain pending under $A_{\mathrm{R}}(z)$ at that acknowledgment, they are served within $z/q$ further time, giving a charge of at most $z$.

On the other hand, at least $\max\{N_1(t)-m,0\}$ completed intervals of $A_{\mathrm{R}}(1)$ contain no acknowledgment of the schedule.
Fix such an interval ending at $b$, and let $a$ be the latest arrival time among its $j$ triggering packets.
The schedule leaves all $j$ packets pending throughout $[a,b]$, paying $j(b-a)\ge1$ during $[a,b]$.
If $q>0$ of them remain pending under $A_{\mathrm{R}}(z)$ at $a$, they are served by $a+z/q$, contributing at most $z$ waiting cost during $[a,b]$.
Thus, on these packets during $[a,b]$, the schedule pays at least $1-z$ more waiting cost than $A_{\mathrm{R}}(z)$.
The triggering packet sets from different intervals are disjoint.
These savings occur while the schedule still has the packets pending and do not overlap with the charges after its acknowledgments.
Consequently,
\begin{align*}
 C_{A_{\mathrm{R}}(z)}(t)-N_z(t)-w
 &\le zm-(1-z)\max\{N_1(t)-m,0\}.
\end{align*}
Rearranging yields $m+w\ge C_{A_{\mathrm{R}}(z)}(t)-N_z(t)+(1-z)N_1(t)$.
Minimizing over schedules proves the claim.

\subsection{Proof of Lemma~\ref{lem:random-comparison}}
\label{app:proof-random-comparison}

Lemma~\ref{lem:two-thresholds} gives $\OPT(t)\ge uN_u(t)+(1-u)N_1(t)$ and $C_{A_{\mathrm{R}}(u)}(t)\le\OPT(t)+N_u(t)-(1-u)N_1(t)$.
These give $N_u(t)\le\OPT(t)/u$ and
\begin{align*}
 2\OPT(t)-C_{A_{\mathrm{R}}(u)}(t)
 &\ge\OPT(t)-N_u(t)+(1-u)N_1(t)
 \ge(1-u)\bigl(2N_1(t)-N_u(t)\bigr).
\end{align*}
Combining this with $C_{A_{\mathrm{R}}(u)}(t)\le\OPT(t)+N_u(t)$ proves the deterministic bounds.
For the expectation, average the cost bounds for $u=1$ and $u=x$ to obtain
\begin{align*}
 \E[C_{A_{\mathrm{R}}(Z)}(t)]&\le\OPT(t)+\frac12\bigl(N_x(t)-N_1(t)\bigr)+\frac{1+x}{2}N_1(t)
\le\left(1+\max\left\{\frac1{2x},\frac{1+x}{2}\right\}\right)\OPT(t),
\end{align*}
where the last step uses $N_x(t)\ge N_1(t)$ and $x(N_x(t)-N_1(t))+N_1(t)\le\OPT(t)$.

\subsection{Proof of Lemma~\ref{lem:expected-switching-learner}}
\label{app:proof-expected-switching-learner}

Let $\nu_{\mathrm{L}}$ count redraws in the first $h$ rounds, including those that select the same expert.
Set $L=\sum_{s=1}^h\ell_s(I_s)$, $W_s=\sum_{i=1}^k w_{i,s}$, and $p_{i,s}=w_{i,s}/W_s$.
The synchronization rule gives $Y\le L+\nu_{\mathrm{L}}$.

We first show by induction that $\Pr(I_s=i)=p_{i,s}$.
This holds for $s=1$.
The retention and redraw rules give
\begin{align*}
 \Pr(I_{s+1}=i)
 &=p_{i,s}(1-\eta)^{\ell_s(i)/2}
 +\left(1-\frac{W_{s+1}}{W_s}\right)p_{i,s+1}
 =p_{i,s+1},
\end{align*}
completing the induction.

Let $\bar\ell_s=\sum_{i=1}^k p_{i,s}\ell_s(i)$ and $\bar L=\sum_{s=1}^h\bar\ell_s=\E[L]$.
Since $\ell_s(i)/2\in[0,1]$, convexity gives
\begin{align*}
 (1-\eta)^{\ell_s(i)/2}&\le1-\frac{\eta}{2}\ell_s(i).
\end{align*}
Consequently, $W_{s+1}/W_s\le1-\eta\bar\ell_s/2$.
For every $i\in[k]$,
\begin{align*}
 \frac1k(1-\eta)^{L_i/2}
 &=w_{i,h+1}\le W_{h+1}
 \le\prod_{s=1}^h\left(1-\frac{\eta}{2}\bar\ell_s\right)
 \le e^{-\eta\bar L/2}.
\end{align*}
Taking logarithms and using $-\ln(1-\eta)\le\eta+\eta^2$ yields
\begin{align*}
 \bar L&\le(1+\eta)L_i+\frac{2\ln k}{\eta}.
\end{align*}

Conditional on $I_s=i$, the redraw probability is
\begin{align*}
 1-(1-\eta)^{\ell_s(i)/2}
 &\le-\frac{\ell_s(i)}2\ln(1-\eta)
 \le\frac{\eta+\eta^2}{2}\ell_s(i)
 \le\eta\ell_s(i).
\end{align*}
Averaging and summing gives $\E[\nu_{\mathrm{L}}]\le\eta\bar L$.
Therefore,
\begin{align*}
 \E[Y]-L_i
 &\le(1+\eta)\bar L-L_i
 \le(2\eta+\eta^2)L_i+\frac{2(1+\eta)\ln k}{\eta}
 \le4\eta L_i+\frac{4\ln k}{\eta},
\end{align*}
where the last inequality uses $\eta\le1/2$.

\end{document}